\documentclass[nohyperref]{article}

\usepackage{microtype}
\usepackage{graphicx}
\usepackage{subfigure}
\usepackage{booktabs} 
\usepackage{mathtools}
\usepackage{url}
\usepackage{enumitem}

\usepackage{hyperref}

\usepackage[accepted]{icml2022}

\usepackage{amsmath}
\usepackage{amssymb}
\usepackage{mathtools}
\usepackage{amsthm}

\usepackage[capitalize,noabbrev]{cleveref}

\theoremstyle{plain}
\newtheorem{theorem}{Theorem}[section]

\newtheorem{lemma}[theorem]{Lemma}

\theoremstyle{definition}
\newtheorem{definition}[theorem]{Definition}
\newtheorem{assumption}[theorem]{Assumption}
\theoremstyle{remark}

\usepackage[textsize=tiny]{todonotes}

\newcommand{\fairnesspapertitle}{Towards A Scalable Solution for \\Improving Multi-Group Fairness in Compositional Classification}

\newcommand{\yhat}{\widehat{y}}
\newcommand{\Yhat}{\widehat{Y}}

\icmltitlerunning{Towards A Scalable Solution for Improving  Multi-Group Fairness in Compositional Classification}

\begin{document}

\twocolumn[
\icmltitle{\fairnesspapertitle}


\icmlsetsymbol{equal}{*}

\begin{icmlauthorlist}
\icmlauthor{James Atwood}{goog}
\icmlauthor{Tina Tian}{goog}
\icmlauthor{Ben Packer}{goog}
\icmlauthor{Meghana Deodhar}{goog}
\icmlauthor{Jilin Chen}{goog}
\icmlauthor{Alex Beutel}{openai}
\icmlauthor{Flavien Prost}{goog}
\icmlauthor{Ahmad Beirami}{goog}
\end{icmlauthorlist}

\icmlaffiliation{goog}{Google}
\icmlaffiliation{openai}{OpenAI (work done while at Google)}

\icmlcorrespondingauthor{James Atwood}{atwoodj@google.com}
\icmlcorrespondingauthor{Tina Tian}{ttian@google.com}
\icmlcorrespondingauthor{Ahmad Beirami}{beirami@google.com}

\icmlkeywords{Machine Learning, ICML}

\vskip 0.2in
]



\printAffiliationsAndNotice{}  

\begin{abstract}
Despite the rich literature on machine learning fairness, relatively little attention has been paid to remediating complex systems, where the final prediction is the combination of multiple classifiers and where multiple groups are present. In this paper, we first show that natural baseline approaches for improving equal opportunity fairness scale linearly with the product of the number of remediated groups and the number of remediated prediction labels, rendering them impractical. We then introduce two simple techniques, called {\em task-overconditioning} and {\em group-interleaving}, to achieve a constant scaling in this multi-group multi-label setup. Our experimental results in academic and real-world environments demonstrate the effectiveness of our proposal at mitigation within this environment.\vspace{-.2in}
\end{abstract}
\section{Introduction}
The literature around group fairness is relatively rich when we consider a binary classifier and desire to satisfy group fairness for a (binary) group~\cite{dwork2012fairness, kamiran2010discrimination, zafar2017fairness, beutel2017data, agarwal2018reductions, donini2018empirical, mary2019fairness, prost2019toward, baharloueirenyi, cho2020fair, lowy2022stochastic}. 
However, many real-world applications go beyond a single binary decision and we are often faced with multi-label systems where the end decision is a composition of the individual labels~\cite{dwork2018individual,adomaviciustoward,burke2002hybrid,he2014practical,wang2011cascade}. 

In this paper, we study a multi-label classification system where several binary classification decisions are combined to make a final prediction for any given input. We consider a special composite classifier where the overall system decision is $1$ if any of the individual binary classifier outputs are~$1$.  For example, consider a content moderation system for an online forum that predicts whether a given comment is toxic, insulting, or attacking identity, and hides a comment if any of the predictions are positive~\cite{pavlopoulos2020toxicity}. {\em How do we perform classification based on the combination of these individual predictions, and achieve specific group fairness goals?} This problem, which is a special instance of the compositional fairness~\cite{dwork2018individual, wang2021practical}, is the focus of this paper.
In addition, we also study a multi-group setting where we are interested in the fairness of the classifier system with respect to many groups

Among the in-process mitigation techniques~\citep{ristanoski2013discrimination, quadrianto2017recycling,kamiran2010discrimination,raff2018fair, aghaei2019learning,donini2018empirical,fish2015fair,grari2020learning, cho2020kde, prost2019toward, lowy2022stochastic,zafar2017fairness, berk2017convex, taskesen2020distributionally, chzhen2020minimax, baharlouei2019rnyi, jiang2020wasserstein, grari2019fairness}, we focus our mitigation strategy on the MinDiff technique~\cite{beutelchallenges,prost2019toward} for improving equality of opportunity~\cite{hardt2016equality} in classifiers. MinDiff has proven to be effective at inducing equality of opportunity while maintaining overall classifier performance across a variety of tasks by relying on the maximum mean discrepancy (MMD) estimators~\cite{gretton2012kernel, prost2019toward}. Importantly, MinDiff can be successfully applied to environments when instances labeled with group membership are very sparse, by using a dedicated data streams to ensure that each mini-batch contains a constant number of group labeled examples. 

A natural baseline in this scenario is an extension of MinDiff to fairness mitigation in multigroup, multilabel environments, where one regularizer is introduced per each group and per each classifier. However, this baseline causes the batch size to scale linearly in both the number of groups and number of prediction tasks being remediated. Even for small number of groups and and small number of classifiers, this can quickly grow out of hand to the extent that this baseline becomes impractical, especially when the baseline classifier is already expensive to train. 
This \emph{significantly} increases resource usage and slows training as the number of groups and prediction tasks grows. 
We propose two simple optimization techniques to achieve this  fairness goal with a constant scaling, with empirical verification. Our contributions are summarized below:\vspace{-.1in}
\begin{itemize}[itemsep=0pt]
    \item {\em Task-overconditioning:} The natural extension of MinDiff requires a batch of negative examples for each label, resulting in a constant scaling with the number of classifiers. Instead, {\em task-overconditioning} suggests using a single batch that contains the negative examples across all labels.
    We argue that task-overconditioning further aligns the overall optimization objective to that of mitigating the overall compositional decision, which is our goal, while also achieving a constant scaling with the number of individual classifiers.
    \item {\em Group-interleaving:} The natural extension of the mitigation solution requires a batch of negative examples with respect to each group at each iteration. Instead, {\em group-interleaving} makes the optimization objective stochastic with respect to groups at each iteration, allowing a constant scaling with the number of groups. 
    \item {\em Empirical verification:} We empirically show that our proposed method, which combines overconditioning and group-interleaving results in equal or better Pareto frontiers than baseline methods, with significant training speedup, on two datasets.
\end{itemize} 

{\bf Related Work.}
Methods for improving group fairness can generally be categorized in three main classes: \emph{pre-processing}, \emph{post-processing}, and \emph{in-processing} methods. Pre-processing algorithms \citep{feldman2015certifying, zemel2013learning,calmon2017optimized} transform the biased data features to a new space in which the labels and sensitive attributes are statistically independent. Post-processing approaches \citep{hardt2016equality, pleiss2017fairness, alghamdi2022beyond} achieve group fairness properties by altering the final decision of the classifier. 

The focus of this paper is on in-processing methods, which introduce constraints/regularizers for improving fairness in training. These methods have empirically shown to produce a more favorable performance/fairness Pareto tradeoff compared to other methods~\cite{lowy2022stochastic}. These include~\citep{ristanoski2013discrimination, quadrianto2017recycling} decision-trees \citep{kamiran2010discrimination,raff2018fair, aghaei2019learning}, support vector machines \citep{donini2018empirical}, boosting~\citep{fish2015fair}, neural networks \citep{grari2020learning, cho2020kde, prost2019toward, lowy2022stochastic},  or (logistic) regression models~\citep{zafar2017fairness, berk2017convex, taskesen2020distributionally, chzhen2020minimax, baharlouei2019rnyi, jiang2020wasserstein, grari2019fairness}.
See the recent paper by~\cite{hort2022bia} for a more comprehensive literature survey. We focus in this paper on the MinDiff technique, which has been successful across tasks \cite{beutelchallenges,prost2019toward,beutel2019fairness}.

This paper is also broadly related to the compositional fairness literature~\cite{dwork2018individual,dwork2020individual,wang2021practical}. In contrast to these works, we focus on a narrower sense of compositionality (only the intersection) for which we derive a scalable specialized solution.

\vspace{-.1in}
\section{Background \& Problem Setup}
Here, we formally provide the problem setup. Let $(x, \{y_t\}_{t \in [T]})$ represent a feature and a set of $T$ binary labels, where $x \in \mathcal{X}$, and $y_t \in \{0, 1\}$ for all $t \in [T]$\footnote{We define [T]:= \{1, \ldots, T\}.}. In our setup, the 
{\em overall decision} is a simple composite function of the individual labels: $
    y = \max_{t \in T} \{y_t\},
$
i.e., $y = 1$ if and only if there exists $t \in [T]$ s.t. $y_t = 1.$

We consider a scenario where we train $T$ individual predictors $\{\yhat_t(x; \theta)\}_{t \in [T]}$ in a multi-label setup, where  $\yhat_t(x; \theta) \in \{0, 1\}$ is a binary classifier from features $x$, and $\theta$ represents model parameters.\footnote{We often drop $(x; \theta)$ for brevity and refer to the output of the $t$-th classifier as $\yhat_t$.} Similarly, the {\em overall model prediction} is given by
$
    \yhat = \max_{t \in T} \{\yhat_t\},
$
i.e., $\yhat = 1$ if and only if there exists $t \in [T]$ s.t. $\yhat_t = 1.$ In other words, we predict that the overall label is $1$ when any of the underlying classifiers is triggered. As explained before, this setup is common in many applications, where the final decision (e.g. rejecting a comment \cite{dixontext}) based on $\yhat$ depends on many sub-decisions $\yhat_t$ (properties of the customer or comment).

Our goal is to optimize for fairness in the equal opportunity sense~\cite{hardt2016equality} for the overall model prediction with respect to multiple group memberships.\footnote{Fairness of  individual predictors is desirable but not required.} Let the set $\mathcal{G}$ capture all groups for which we would like to improve fairness. Let $g_m \in \{0, 1\}$ denote the identifier for membership in group $g_m,$ for $m \in [|\mathcal{G}|].$ 
\begin{definition}[overall equal opportunity with respect to membership in group $m$]
\label{def:eq_bias}
\begin{equation}
\label{eqn:eq_bias}
    P(\Yhat = 1 | G_m = 0, Y = 0) = P(\Yhat = 1 | G_m = 1, Y = 0), 
\end{equation}
Note that in this paper, we do not consider the intersectional fairness setting~\cite{kearns2018preventing, foulds2020intersectional} where the goal is to ensure fairness to all intersections of group memberships; see Appendix~\ref{sec:limitations}.
\label{obj-e2e}
\end{definition}
While there are numerous ways to optimize for fairness in machine learning, specifically in equal opportunity sense, the methods that achieve better fairness/performance Pareto frontiers have been empirically observed to be mostly in-processing methods~\cite{lowy2022stochastic}, where a regularizer is added to the (cross-entropy) training loss to mitigate the model fairness gap. The regularizer is usually of the form:
$
    D(\Yhat(\theta) , G_m | Y= 0) 
$
 where $D(\cdot, \cdot)$ is a proper divergence between two random variables.

Notice that in our problem setup $\Yhat(\theta)$ is not a differentiable function of the task-level predictors. Hence, we cannot use it directly to regularize the training of the individual task-level classifiers via backpropagation. This situation occurs when task-level predictors are not trained jointly or are even owned by different teams in an organization.

One intuitive solution to remediate this multi-label setup is to ensure that each individual classifier is fair for each group \cite{dwork2018individual}. This intuitive design is motivated by previous work \cite{wang2021practical} which finds that fairness of individual predictors might be sufficient to improve fairness of the overall system, even if there are no theoretical guarantees. We refer to this objective as task-level equal opportunity:
\begin{definition}[Task-level equal opportunity with respect to group $G_m$]
For any task $t \in [T]$ and group $m$ for $m \in |\mathcal{G}|$,
\begin{equation}
    P(\Yhat_t = 1 | G_m = 0, Y_t = 0) = P(\Yhat_t = 1 | G_m = 1, Y_t = 0).
\label{eqn:conditionals}
\end{equation}
\vspace{-.25in}
\end{definition}

\section{Baseline: Many MinDiff Regularizers}
\label{sec:baseline}
While there are many effective methods for solving task-level equal opportunity in \eqref{eqn:conditionals}, as discussed in related work, here we focus on an adaptation of MinDiff~\cite{zafar2017fairness,beutelchallenges,prost2019toward} to the multi-group multi-label classification case.  This regularization-based approach has a number of advantages. First, it does not require group labels at inference time, which is often true for real-world applications. Next, it has been empirically demonstrated to be effective at remediating fairness issues while still maintaining overall performance \cite{prost2019toward}.  Finally, it is designed to be effective when group-labeled instances are rare even in training data. 

This MinDiff technique introduces a new loss term based on maximum mean discrepancy (MMD) to promote (conditional) independence between the predictions and sensitive group ~\cite{prost2019toward} per each group and each task.  More precisely, the loss becomes:

\begin{align}
    L_{\textit{MinDiff}} = L_{\textit{CE}}(\hat{Y}, Y) + \lambda \sum\limits_{t \in T}  \sum\limits_{m \in [|\mathcal{G}|]} R_{t,m},
\end{align}
where $L_{\textit{CE}}$ is the empirical cross-entropy loss, $\lambda$ is a hyperparameter that sets the relative strength of the entropy and MMD loss,\footnote{In practice, one can tune the MinDiff strength for each regularizer at the expense of a complex hyperparameter tuning.} and

\begin{equation}
\label{eqn:baselinemindiffR}
    R_{t,m} \!=\! \textit{MMD}(\hat{Y}_t | Y_t \!=\! 0, G_m \!=\! 0; \hat{Y}_t | Y_t \!=\! 0, G_m \!=\! 1).
\end{equation}

Computing $R_{t,m}$ requires negative labeled instances ($Y_t = 0$) for both group membership cases ($G_m = 0$ and $G_m = 1$). In practice, instances with group membership information are much less frequently available than those without. Min Diff handles this by creating dedicated data streams for group-labeled instances that ensure that every batch has the data required to compute the MMD kernel component of the loss. As the number of groups and predictions tasks increases, this leads to $O(T \cdot |\mathcal{G}|)$ data streams that must be stored and a $T \cdot |\mathcal{G}|$ multiplier on the batch size.

\vspace{-.1in}
\section{Proposed Method: MinDiff-IO}
We now describe our proposed method, {\em MinDiff-IO}, which is built on two main components: {\em Group-Interleaving} and {\em Task-Overconditioning.} The central insight behind these two approaches is that we can still accomplish our goal, overall equal opportunity defined by Equation \eqref{eqn:eq_bias}, by optimizing a slightly different objective that is better aligned and is easier to compute. We describe these techniques in the subsequent sections. 

\subsection{Task-Overconditioning}
The baseline method that targets {\em task-level equal opportunity} has a number of data streams and batch size that scales linearly with $T,$ making it intractable for  systems where $T$ might be large (e.g., $O(100)$). Additionally, it does not necessarily imply overall equal opportunity ~\cite{dwork2018individual}, Equation \eqref{eqn:eq_bias}, which is what is desired. 

In this section, we present our proposal towards satisfying overall equal opportunity in this compositional decision system. We provide limited theoretical motivation for why it might be more aligned with the overall fairness objective under restrictive assumptions. We shall also see in the experimental section (where those restrictive assumptions are not satisfied) that it leads to equal or better fairness/performance Pareto frontiers. 

\begin{definition}[Overconditioning task-level equal opportunity]
 For all $t \in [T],$
\begin{equation}
    P(\widehat{Y}_t \!=\! 1 | G_m \!=\! 0, Y \!=\! 0) \!=\! P(\widehat{Y}_t \!=\! 1 | G_m \!=\! 1,\! Y \!=\! 0). 
\end{equation}
\vspace{-0.25in}
\label{obj-oc}
\end{definition}
Note that, unlike Equation \eqref{eqn:baselinemindiffR}, we condition on all labels having negative truth. This has the effect of requiring only one dataset for all tasks when computing loss rather than $T$ datasets.
\begin{assumption}
Let task-level classifiers $\{\yhat_t(x; \theta)\}_{t \in [T]}$ be such that for all $x \in \mathcal{X},$ and for any $t \neq \tau,$ 
\begin{equation}
    \yhat_t(x; \theta) \yhat_\tau(x; \theta) = 0.
\end{equation}
In other words, the classifiers don't trigger simultaneously; if $\yhat_t = 1$ then $\yhat_\tau = 0$ for all $\tau \neq t.$
\label{assump1}
\end{assumption}

Notice that Assumption~\ref{assump1} is a strong assumption as it requires the classifiers to have non-overlapping coverage, which is not necessarily satisfied in practice. For example, in the content moderation example, a comment might be toxic, insulting, and attacking identity at the same time. While this assumption is very restrictive, we show that under this scenario overconditioning is perfectly aligned with the goal of mitigating overall classifier. We also don't need this assumption for our empirical results, which show improvements over the baseline classifier.
\begin{lemma}
If Assumption~\ref{assump1} is satisfied, then Definition~\ref{obj-oc} (overconditioned task-level equal opportunity) implies Definition~\ref{obj-e2e} (overall equal opportunity).
\label{lem:oc-implies-e2e}
\end{lemma}
The proof is relegated to the appendix. Lemma~\ref{lem:oc-implies-e2e} determines a scenario where overconditioning task-level equal opportunity indeed implies the desired overall equal opportunity. Notice that even under Assumption~\ref{assump1}, Definition~\ref{obj-oc} is a stronger requirement than Definition~\ref{obj-e2e}, and is not implied by it. In other words, we might be able to satisfy the overall equal opportunity and yet the overconditioning equal opportunity might not be satisfied for all task-level classifiers.

To solve for Task-Overconditioning, we adapt MinDiff loss as follows:
\begin{align}
        L_{\textit{MinDiff-O}} = L_{\textit{CE}}(\hat{Y}, Y) +
        \lambda \sum_{t \in T} \sum_{m \in [|\mathcal{G}|]}  R^{\textit{O}}_{t,m},
\end{align}
where
\begin{align}
   R^{\textit{O}}_{t,m} \!=\! \textit{MMD}(\hat{Y}_t | Y \!=\! 0, G_m \!=\! 0; \hat{Y}_t | Y \!=\! 0, G_m \!=\! 1).
   \label{eqn:ROvercondition}
\end{align}

Note that there will be fewer data instances that are suitable for computing \eqref{eqn:ROvercondition}, which requires all labels to be jointly negative, than \eqref{eqn:baselinemindiffR}, which only requires individual labels to be negative. We have not found this to be an issue in practical applications where positive label incidence is low. 

\subsection{Group-Interleaving}
\label{sec:interleave}
MinDiff was originally designed to present remediation data from all groups to the model at each iteration. However, we can reduce the complexity of computing the MinDiff regularizer further by presenting only one group per batch to the model. In this case, the loss becomes:
\begin{equation}
 L_{\textit{MinDiff-IO}} = L_{\textit{CE}}(\hat{Y}, Y) + R^{\text{O}}_{M}  
\end{equation}
where $M$ is a random index supported on $[|\mathcal{G}|].$ In other words, here we remediate against a random draw from the groups at each iteration of the algorithm. Notice that the new loss is the same as the task-overconditioned loss in expectation, and is expected to converge to a stationary point of the same objective. On the other hand, when combined with Task-Overconditioning, the loss can be computed with only $O(1)$ extra instances in each batch, with no dependence on $|\mathcal{G}|$ and $T$.

\section{Evaluation Metrics}

For each binary group membership, i.e., $G_m \in \{0, 1\},$ where $G_m = 1$ is considered the minority group membership, we quantify the fairness gap through the following interchangeable metrics that are expressed in terms of the absolute gap and the ratio of the two groups:
\begin{equation}
     d_{EO, m} = |\textit{FPR}_{G_m = 1} - \textit{FPR}_{G_m = 0}|,
\end{equation}
and
\begin{equation}
    r_{EO, m} = \textit{FPR}_{G_m = 1} / \textit{FPR}_{G_m = 0},
\end{equation}
where $\widehat{P}$ denote the empirical distribution over a test set of $N$ i.i.d samples from $P_{XY}$, and for $i \in \{0, 1\}$,
$
    \textit{FPR}_{G_m = i} := \widehat{P}(\widehat{Y} = 1 | G_m = i, Y = 0).
$

To measure the classification performance we both compute the Area Under the ROC Curve (ROC AUC) of the classifier as well as accuracy.
Finally, to measure speed, we report the number of iterations per second achieved during model training.
\newcommand{\systemfigscale}{0.5}

\section{Experiments}
\label{sec:experiments}
\begin{table*}[t]
    \centering
    \resizebox{0.7\textwidth}{!}{
    \begin{tabular}{c|c|c|c|c}
         Method & Avg Steps / Sec ($\uparrow$) &  Avg AUCPR ($\uparrow$) &  $r_{EO, 1}$ ($\downarrow$)&  $r_{EO, 2}$ ($\downarrow$)\\
         \hline
         No remediation & $A$ & $B$ & $C$ & $D$ \\
         MinDiff & -- & -- & -- & -- \\
         MinDiff-O & $0.60 \times A$ & $0.96 \times B$ & $ 0.66 \times C$& $ 0.79 \times D$ \\
         MinDiff-IO & $0.86\times A$ &$0.96 \times B$ & $ 0.65 \times C$ & $ 0.83 \times D$ \\
    \end{tabular}}
    \caption{\small Comparison of different remediation techniques for the policy classification model.
    Avg AUCPR is a summary measurement of the area under the precision-recall curve for each policy classifier; and Group 1 and Group 2 FPR ratio ($r_{EO, 1}$ and $r_{EO, 2}$, respectively) denote the ratio of minority to baseline false positive rates at the system level. The baseline MinDiff remediation is too slow to run at this scale (second row). We show how the introduction of Task Overconditioning allows us to remediate at all (third row) and how adding Group Interleaving reduces the speed cost incurred by remediation (fourth row).}
    \vspace{-.2in}
    \label{tab:policy_table}
\end{table*}
We run two experiments.  The first experiment provides the Pareto frontier of fairness vs performance for each approach using a publicly-available academic dataset, and the second provides the performance, fairness, and speed of a real-world policy enforcement classifier at a particular operating point with each of the proposed approaches.  Overall, these experiments show that MinDiffIO provides equal or better fairness/performance while improving training speed.

\begin{figure}[t]
    \centering
    \includegraphics[scale=\systemfigscale]{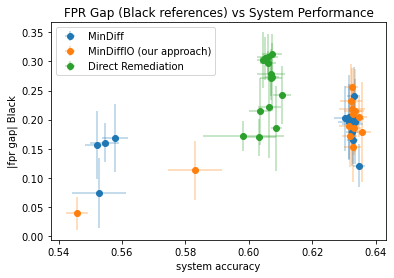}
    \vspace{-.1in}
    \caption{\small System-level tradeoff curve for the civil comments dataset, achieved by  varying regularizer strength $\lambda$. Note that direct remediation is outperformed by component-based approaches; for a given level of fairness, the component-based approaches provide better performance. Both component-based approaches offer similar performance. Five runs were performed for each $\lambda$ and the faint cross-hatches represent 95\% confidence intervals.}
    \vspace{-.15in}
    \label{fig:civilcommentssystem}
\end{figure}

\subsection{Civil Comments}
The first set of experiments are run on the Civil Comments Dataset~\cite{pavlopoulos2020toxicity}; details are given in Appendix \ref{sec:detailedcivilcomments}. Civil Comments contains comment text and seven associated crowd annotated labels related to the `civility' of the comment; whether the comment is an insult, toxic, or attacking identity, and so on. We use the subset of the data that are labeled with group information. Groups are related to race, ethnicity, gender, disability, and sexuality.

We train comment classifiers on three of the seven labels and combine the predictions into a system-level prediction: a comment is classified as unsafe is any of the prediction is unsafe. We compare this with `direct remediation' where a classifier is trained to predict the system-level label (the logical OR of the component labels) rather than the components. In addition, we compare with `component-based' remediation where MinDiff or MinDiffIO are applied to component classifiers. Results are shown for one group (Black) in Figure \ref{fig:civilcommentssystem}; results for other groups and component-level results for all predictions and groups can be found in Appendix \ref{sec:detailedcivilcomments}.

The plot displays the tradeoff between fairness and performance as the hyperparameter $\lambda$ is varied. As $\lambda$ increases, the contribution of the MMD component of the loss grows, leading to increased fairness at the cost of performance.

We observe that the component-based approaches offer better performance for a given fairness than direct remediation. Also, MindDiff and MinDiffIO offer qualitatively similar results.

\subsection{Product Policy Compliance Detection}
We now study a real world system, which is responsible for filtering out examples which break the product policy. This is similar to literature on toxic comment detection \cite{pavlopoulos2020toxicity} or hateful speech filtering \cite{dixontext}. To reflect the different facets of the product, a set of rules (10-1000) are defined and an example is against product policy if any given rule is broken. In practice, we use individual classifiers to predict each rule, and an example is filtered out if any individual soft prediction reaches a certain threshold.  Samples can be categorized into two sensitive attributes (each considered as binary) and we want to guarantee fairness to samples from each group, which lends itself to  the multi-label and multigroup classification.

Note that a false positive of this system is a user harm because policy-following content is flagged as policy-violating. Our goal is to reduce the gaps between false positive rates between minority groups and a baseline population.

We first evaluate the initial system without any remediation and find that two groups have high false positive rate differences.  Our goal is to design the mitigation strategy that reduces the observed gaps on the final policy (gap from Definition~\ref{def:eq_bias}) for both groups, while  maintaining good performance (measured by AUCPR) and training speed (training steps/sec).

In Table \ref{tab:policy_table}, we show four different remediation approaches. The first approach, unremediated, has some performance, fairness, and speed characteristics that we compare other approaches to. The second approach, baseline, is unworkably slow, so we are unable to run experiments or provide results. The third approach, which introduces Task Overconditioning, reduces fairness gaps with a minor hit to performance and a major hit to speed. Finally, the fourth approach adds Group Interleaving to mitigate the speed impact while maintaining similar fairness and performance characteristics.

\section{Conclusion}
Prior in-process equal opportunity remediation methods suffer from poor (linear) scaling in the number of prediction tasks and number of groups to remediate, making existing techniques sometimes impossible to apply to real-world scenarios. We present Mindiff-IO, a new method that builds on the MinDiff approach to provide constant scaling with respect to tasks and groups. We show that Mindiff-IO provides similar performance and fairness characteristics to MinDiff while scaling much better in multilabel and multigroup environments through experiments with both academic and real-world datasets. 
The limitations of this work are provided in Appendix~\ref{sec:limitations}.

\section*{Acknowledgements}
We would like to thank Preethi Lahoti, Ananth Balashankar, Lucian Cionca, and Katherine Heller for their constructive feedback on this paper.

\bibliographystyle{icml2022}
\bibliography{main}

\appendix
\onecolumn

\section{Limitations}
\label{sec:limitations}
There are three limitations to the approaches mentioned here. First, overconditioning requires instances that have negative ground truth for all modeled labels\footnote{Note that the ground truth negative requirement is present when optimizing for equality of opportunity with respect to false positive rates. If equality of opportunity with respect to false negative rates were the goal, the method would instead require ground truth positives.} in order to compute the Min Diff loss. This is a realistic environment; for instance, a policy dataset where policy-violating content is rare. However, if true positives are very common, this method may no longer be effective.

The second limitation is with respect to intersectional group fairness. The interleaving optimization described in Section \ref{sec:interleave} does not explicitly represent or remediate the intersection of groups. Intersectional remediation is a more difficult problem due to the exponential scaling of the number of intersections with respect to the number of groups. We opted not to remediate intersections because of the sparsity of our group labels - very few instances are labeled with more than one group. We believe that techniques that effectively and efficiently address intersectional remediation are an interesting area for future work.

Third, we only consider MinDiff-based techniques in this paper and demonstrate that Mindiff-IO has better scaling characteristics than the original MinDiff approach. Future work could compare the fairness, performance, and scaling properties of Mindiff-IO with other methods of achieving equal opportunity. In addition, future work could test the application of interleaving and overconditioning to other in-processing methods.

\section{Proofs, Experiment Details, and Further Results}
\subsection{Proof of Lemma~\ref{lem:oc-implies-e2e}}
\begin{proof}[Proof of Lemma~\ref{lem:oc-implies-e2e}]
The proof is completed by noting that
\begin{align}
    P(\widehat{Y} = 1 | G = 0, Y = 0) &= 
    P(\max_{t \in [T]} \widehat{Y}_t = 1| G = 0, Y = 0)\nonumber \\
    & = \sum_{t \in [T]} P(\widehat{Y}_t = 1 | G = 0, Y = 0) \label{eq:assump1}\\
    & = \sum_{t \in [T]} P(\widehat{Y}_t = 1 | G = 1, Y = 0) \label{eq:obj-oc} \\
   & =  P(\max_{t \in [T]}\widehat{Y}_t = 1 | G = 1, Y = 0)  \label{eq:assump1-2} \\
   & = P(\widehat{Y} = 1 | G = 1, Y = 0)\nonumber,
\end{align}
where~\eqref{eq:assump1} follows from Assumption~\ref{assump1}, and~\eqref{eq:obj-oc} follows from Definition~\ref{obj-oc}, and \eqref{eq:assump1-2} follows from Assumption~\ref{assump1}. 
\end{proof}

\subsection{Civil Comments Experimental Details}
For these experiments we select three labels (identity attack, insult, and toxicity) as well as four groups (black, gay or lesbian, female, and transgender) for modeling and remediation. Our model consists of a single hidden layer deep neural network that takes a simple hashing trick bag of words vectorization of the comment text as input. The hidden layer and text vector have 64 and 1,000 elements, respectively.

All models are trained for 25 epochs with a learning rate if 0.1 and a Gaussian kernel weight of 1.0.

We present empirical Pareto frontiers of fairness (here, the absolute value of the difference between false positive rates for a group and baseline) and performance (here, ROC AUC). Thresholds for the fairness dimension are selected through calibration on a validation set.

\begin{figure*}[!htb]
\centering
\includegraphics[scale=\systemfigscale]{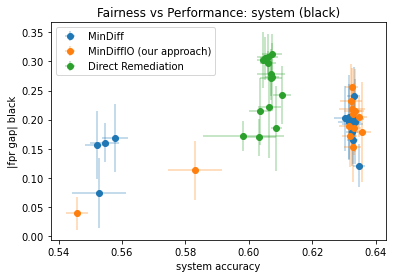}
\includegraphics[scale=\systemfigscale]{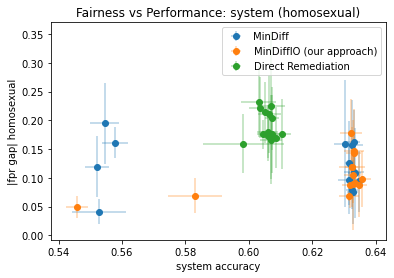}
\includegraphics[scale=\systemfigscale]{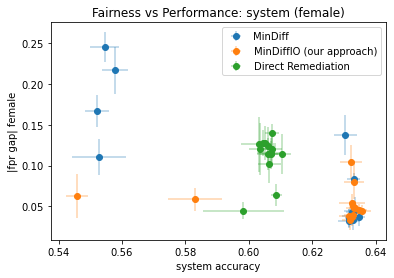}
\includegraphics[scale=\systemfigscale]{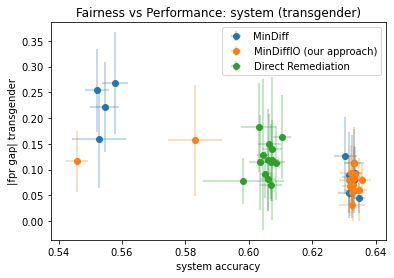}
\caption{System-level results for all four groups. Note that, in each case, component-based techniques outperform direct remediation by offering a higher performance for a given fairness.}
\label{fig:civilcommentsfullsystem}
\end{figure*}

\newcommand{\figscale}{0.4}
\begin{figure*}[!htb]
    \centering
    \includegraphics[scale=\figscale]{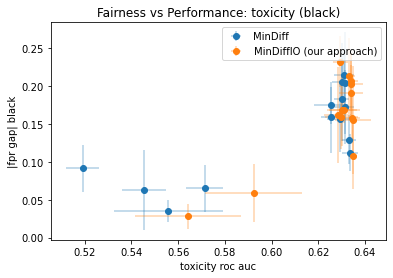}
    \includegraphics[scale=\figscale]{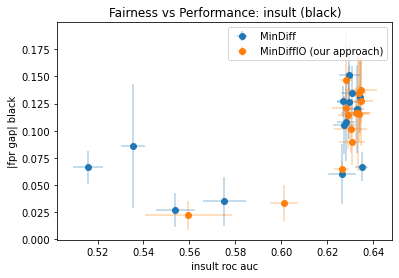}
    \includegraphics[scale=\figscale]{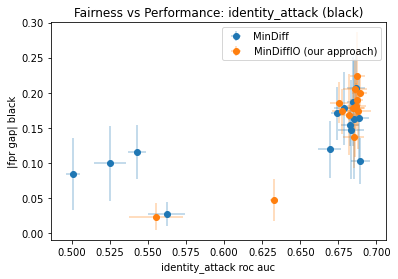}
    \includegraphics[scale=\figscale]{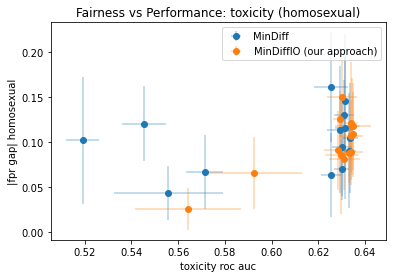}
    \includegraphics[scale=\figscale]{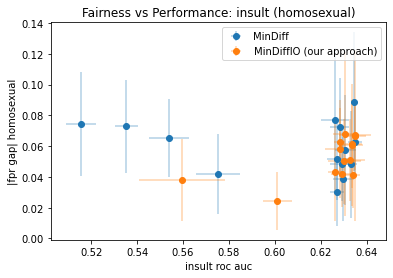}
    \includegraphics[scale=\figscale]{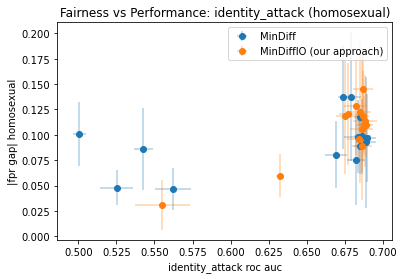}
    \includegraphics[scale=\figscale]{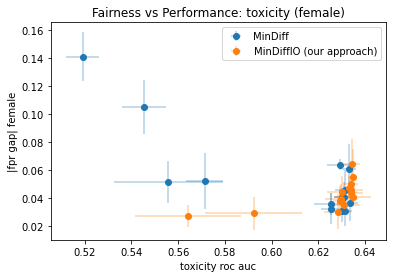}
    \includegraphics[scale=\figscale]{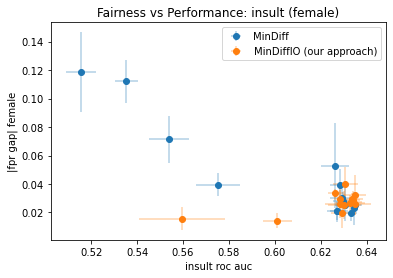}
    \includegraphics[scale=\figscale]{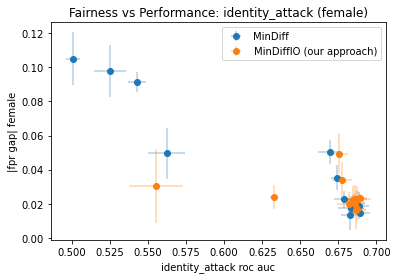}
    \includegraphics[scale=\figscale]{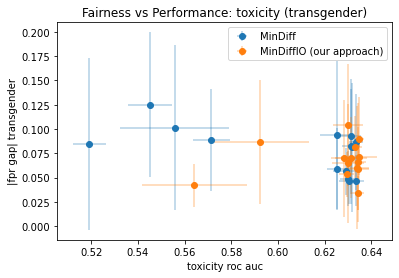}
    \includegraphics[scale=\figscale]{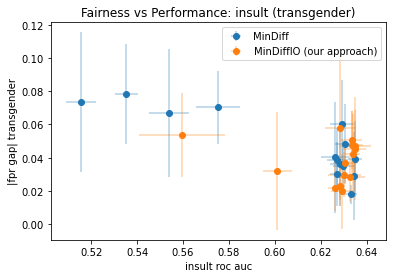}
    \includegraphics[scale=\figscale]{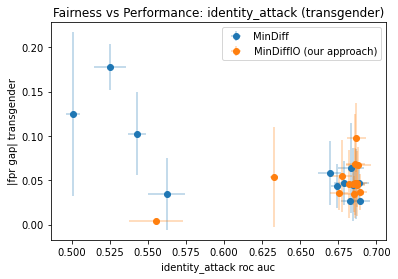}
    \caption{Empirical Pareto frontiers of fairness and performance for baseline MinDiff and newly introduced Mindiff-IO. Each point corresponds to a different min diff strength, and the error bars in each dimension represent 95\% confidence intervals. Note that the frontier achieved by Mindiff-IO not qualitatively different than that achieved by the baseline MinDiff approach. Note that there is a clear frontier that trades off performance and fairness for groups with a large FPR gap (black and homosexual). However, for groups with a low FPR gap (female, transgender), a small min diff strength $\lambda$ provides the best performance and fairness characteristics.}
    \label{fig:civil_comments}
\end{figure*}

\subsection{Detailed Civil Comments Results}
\label{sec:detailedcivilcomments}
System-level results for all four groups are shown in Figure \ref{fig:civilcommentsfullsystem}. Note that, in each case, component-based techniques outperform direct remediation by offering a higher performance for a given fairness.

Component results for each group, label pair are shown in Figure \ref{fig:civil_comments} for both the MinDiff and Mindiff-IO techniques. These Pareto frontiers are generated by varying the hyperparameter $\lambda$, where higher $\lambda$ values put more weight on the MinDiff loss term and lead to improved fairness at the cost of performance. Each data point in the plot is generated by training a model five times; the crosses in each dimension represent 95\% confidence intervals.

Note that each approach achieves a similar Pareto frontier, indicating the Mindiff-IO has similar performance and fairness characteristics. In other words, this experiment confirms that Mindiff-IO does not sacrifice classifier fairness or performance for individual classifiers. In the next experiment, we will provide training speed measurements to demonstrate the scaling advantages of Mindiff-IO.
\end{document}